\documentclass[hidelinks]{article} 
\pdfoutput=1
\usepackage{nips15submit_e,times}

\usepackage{latexsym} 
\usepackage{gensymb}

\usepackage{graphicx} 
\usepackage{subcaption}

\usepackage{algorithm}
\usepackage{algorithmic}

\usepackage{amsmath}
\usepackage{amssymb}
\usepackage{multirow}
\usepackage{multicol}
\usepackage{amsthm}
\newtheorem{theorem}{Theorem}
\newtheorem{definition}{Definition}
 \usepackage{enumitem}

\usepackage{hyperref}
\usepackage{url}

\nipsfinalcopy 

\title{Intrinsic Non-stationary Covariance Function \\for Climate Modeling} 
\author{Chintan A. Dalal\\ 
	Department of Computer Science \\ 
	Rutgers, The State University of New Jersey \\ \href{mailto:chintan.dalal@rutgers.edu}{\nolinkurl{chintan.dalal@rutgers.edu}} \\ \And
 Vladimir Pavlovic\\
 Department of Computer Science \\ 
	Rutgers, The State University of New Jersey \\ \href{mailto:vladimir@cs.rutgers.edu}{\nolinkurl{vladimir@cs.rutgers.edu}} \\ \And
 Robert E. Kopp \\
 Department of Earth \& Planetary Sciences \\
Rutgers, The State University of New Jersey \\
\href{mailto:robert.kopp@rutgers.edu}{\nolinkurl{robert.kopp@rutgers.edu}} }

\begin{document} 



            


\maketitle

\begin{abstract} 

Designing a covariance function that represents the underlying correlation is a crucial step in modeling complex natural systems, such as climate models. Geospatial datasets at a global scale usually suffer from non-stationarity and non-uniformly smooth spatial boundaries. A Gaussian process regression using a non-stationary covariance function has shown promise for this task, as this covariance function adapts to the variable correlation structure of the underlying distribution. In this paper, we generalize the non-stationary covariance function to address the aforementioned global scale geospatial issues. We define this generalized covariance function as an intrinsic non-stationary covariance function, because it uses intrinsic statistics of the symmetric positive definite matrices to represent the characteristic length scale and, thereby, models the local stochastic process. Experiments on a synthetic and real dataset of relative sea level changes across the world demonstrate improvements in the error metrics for the regression estimates using our newly proposed approach.

\end{abstract} 
\section{Introduction}
\label{sec:int}

Covariance functions are a key element in the regression of geospatial data (kriging). Designing a covariance function that can capture the geospatial random field of natural processes is useful for understanding the changes occurring in climate related variables. For example, ongoing sea level changes provide an important context for understanding future coastal flood risks~\cite{church2013sea}. However, most of the climate related datasets at a global scale suffer from statistical issues of non-stationarity and non-uniformly smooth spatial boundaries. In this paper, we design a covariance function which addresses these issues in the datasets.~\looseness-1

Stochastic processes of complex systems at a global scale are known to have a regional geophysical variability. For example, the local sea-level changes occurring near Northern Europe are primarily dominated by the physics of the Glacial-isostatic adjustment, while  the local sea-level changes occurring near Japan are dominated by the tectonics~\cite{kopp2013does}. Such regional variations can be modeled using the non-stationary covariance function (Figures~\ref{fig:Orig}(a,b) depicts these issues).~\looseness-1

One of the important modeling issues in the complex geospatial dataset at a global scale is knowing the boundaries of the regional geophysical variability. Addressing this issue, which we call the non-uniformly smooth spatial boundary issue, aids in modeling the non-stationary covariance function by capturing the underlying correlation structure of a region. This paper addresses this issue for kriging.~\looseness-1

For standard kriging models (GP), the methods given in ~\cite{rasmussen2006gaussian} are widely used in many climate models, including models of the sea level data, because it gives a non-parametric method for analyzing such a geospatial random field. Such a model can be completely defined by its covariance function, and it learns the hyper-parameters of the designed covariance function directly from the training data without making any parametric assumptions about the data.~\looseness-1

To model the non-stationarity of the stochastic process,~\cite{paciorek2004nonstationary,plagemann2008nonstationary,higdon1999non,SteinKrigging} devised a non-stationary covariance function. These methods use spatially evolving smooth kernels to represent the characteristic length scale (CLS) of the covariance function, which, in turn, models the local correlation structure of the global scale stochastic process. From these methods,~\cite{paciorek2004nonstationary} is the most promising approach for climate data, as its model is entirely in the non-parametric GP framework. Even so, it uses a homogeneous set of hyper-parameters to ensure the smoothness of the CLS in the input space. We show in our experiments that using a homogeneous set of hyperparameters is inefficient for the global-scale complex systems.~\looseness-1

Non-stationarity has also been addressed by~\cite{nychka2002multiresolution}. They use a multi-resolution basis function, fixed rectangular grid, and explicit thresholding scheme. The approach of setting fixed grids renders this method unviable for datasets with non-uniformly smooth boundaries. While~\cite{paciorek2004nonstationary} validates their model on precipitation estimates in Colorada (USA),~\cite{nychka2002multiresolution} uses the daily ozone rate in the Midwest. Both of these datasets are locally dense, and, in turn, do not address the global scale climate data issues.~\looseness-1

In the literature of estimating global scale sea-level changes, techniques given by~\cite{hay2015probabilistic} use a Kalman filter model, and~\cite{kopp2013does} use a mixture of GPs at various regional resolutions. However, these methods fail to tightly couple the local and global estimates in the spatial dimension. Additionally, these methods rely on the domain knowledge and a parametric framework.~\looseness-1

In this paper, we propose a covariance function to model the data and address the aforementioned issues at a global scale, and we call this covariance function the intrinsic non-stationary covariance function. We use intrinsic statistics~\cite{amari1987differential} on the space of the CLS to capture the non-uniformly smooth spatial boundaries for the non-stationarity. Furthermore, we provide an algorithm for kriging using the proposed intrinsic covariance function, and we validate its applicability on two global scale datasets.~\looseness-1
  
\begin{figure*}[t]
\centering
\begin{minipage}{0.49\textwidth}
               \includegraphics[width=\linewidth]{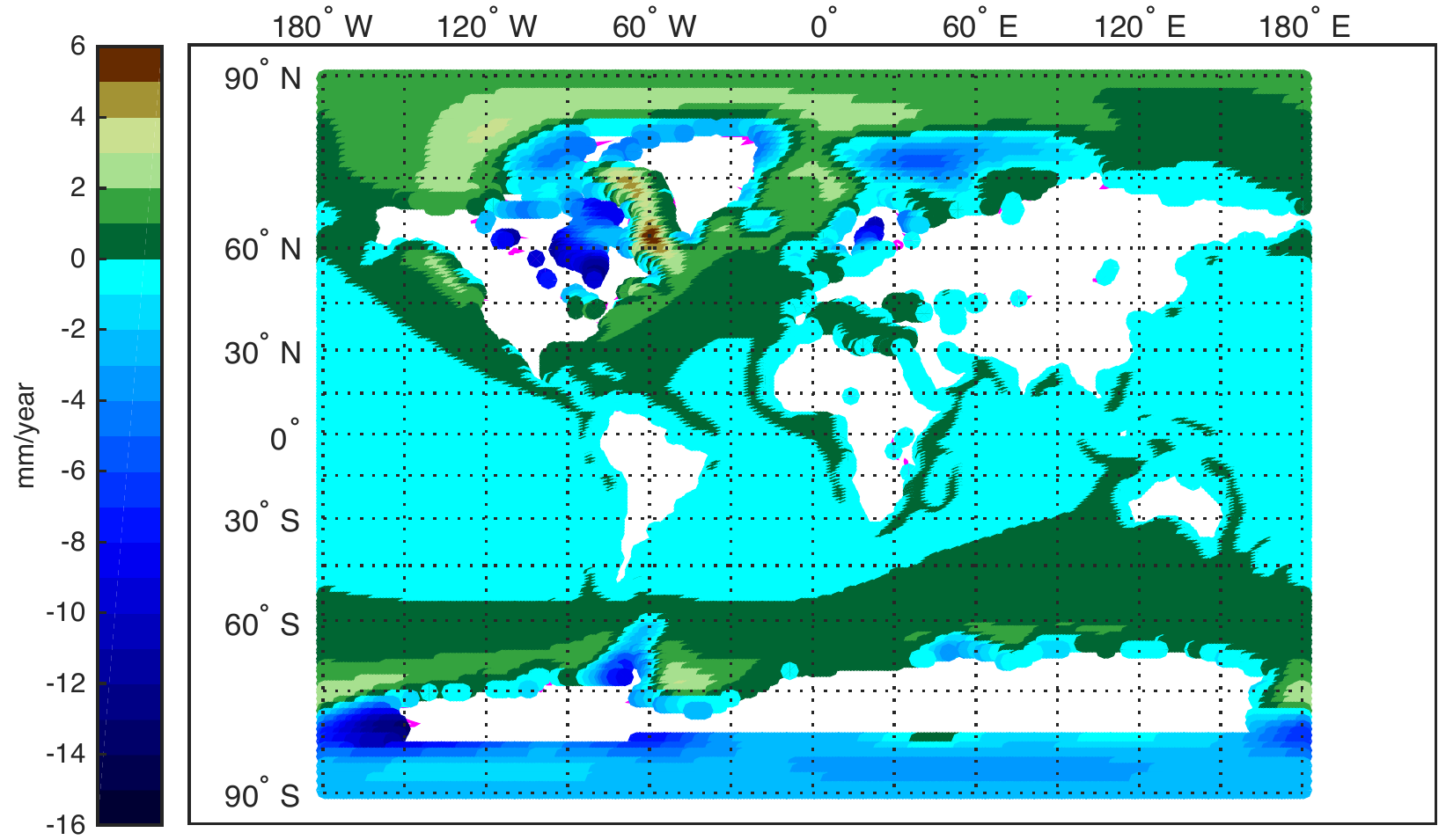}\label{fig:GIA} 
    \subcaption{}
    \end{minipage}
    \hspace{\fill}
    \begin{minipage}{0.49\textwidth}
    \includegraphics[width=\linewidth]{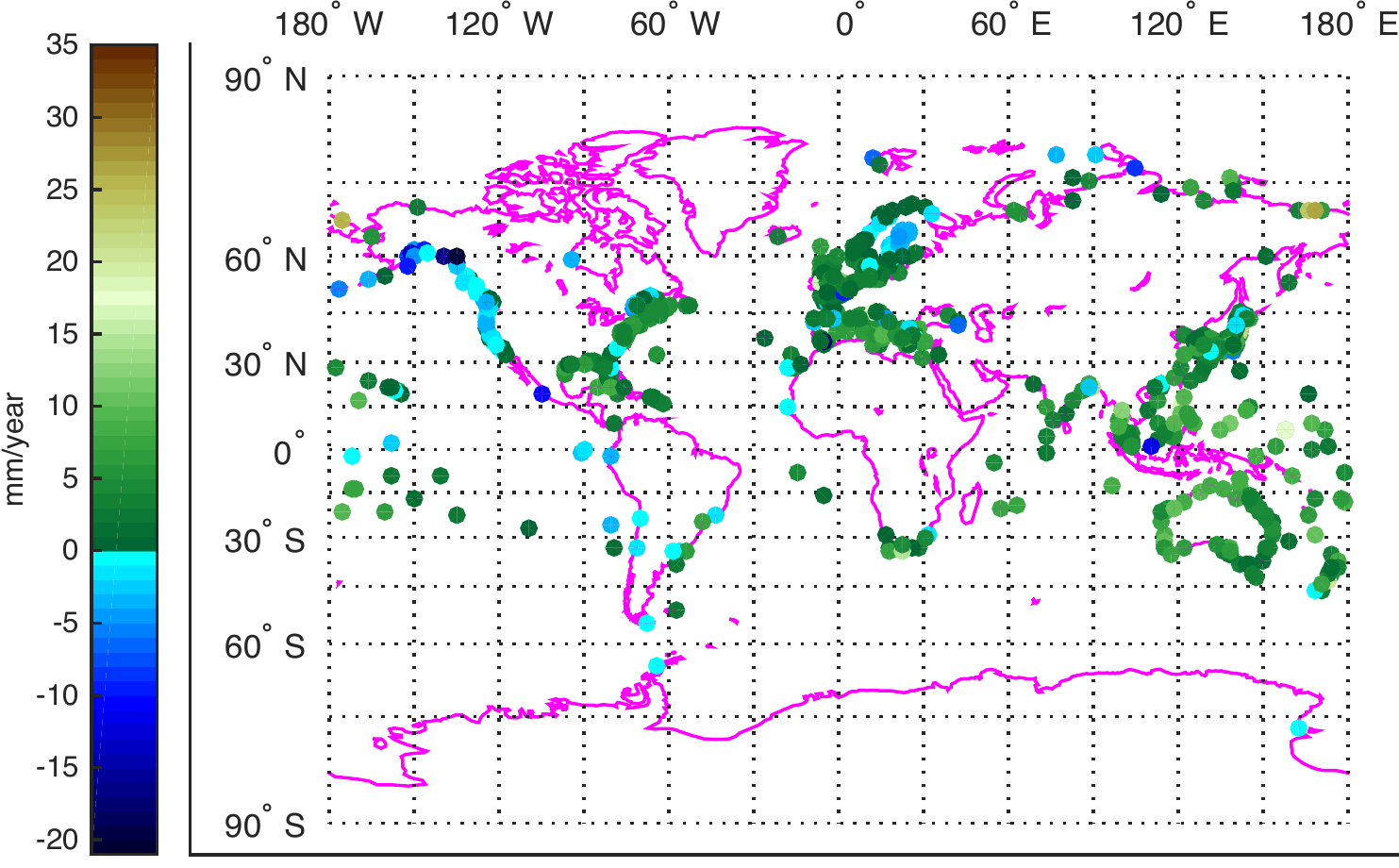}\label{fig:SL}\vspace{-.5em}
    \subcaption{}
    \end{minipage}
     \hspace{\fill}
    \vspace{-1em}
\caption{Experimental datasets. (a) Synthetic data of the relative sea-level change from the Glacial-isostatic adjustment model, and (b) real tide gauge data of the rate of change of the sea level during the last 20 years across the world.}
\label{fig:Orig}
\vspace{-1.5em}
\end{figure*}

\section{Covariance functions used in climate modeling}
In this section, we provide relevant background about stationary and non-stationary covariance functions used in kriging, which helped in the development of our proposed framework of the intrinsic non-stationary covariance function. The standard kriging (GP) model tries to recover the underlying function, $f$, where $y_i \sim f(x_i) + \eta$, from $n$ observed data points $\{(x,y)\}_{i=1}^n$. The learned model is then used to compute the predictive distribution $p(y^* | x^*, f)$ for a test point $x^*$. Here, $x \in \mathbb{R}^d$ are points in the input space of dimension $d$, $y \in \mathbb{R}$ are observed values, and $\eta$ is observed noise. In our problem setting, $x$ represents the 2D spatial coordinates (latitude and longitude) and $y$ represents the rate of change of the sea level at $x$.~\looseness-1

Assuming Gaussian noise $\eta \sim \mathcal{N}(0,\sigma_n^2)$ with a constant variance term $\sigma_n^2$, the predictive distribution is then given by $\mathcal{N}(\mu, \sigma^2)$, where: $\mu = k_{*}^{T}(K + \sigma_{n}^{2}I)^{-1}y$, $\sigma^{2} = k_{**} - k_{*}^{T}(K + \sigma_{n}^{2}I)^{-1}k_{*} + \sigma_{n}^{2}$, $K(i,j) = k(x_i,x_j)$, $k_{*}(i)=k(x*,x_i)$, $k_{**}=k(x*,x*)$, $y = (y_1, \cdots ,y_n)^{T}$, and $I$ is the identity matrix~\cite{rasmussen2006gaussian}. Here, $k(x_i,x_j)$ is the covariance function which we are interested in modeling.~\looseness-1

\subsection{Stationary covariance function}
For kriging, the correlation structure of the data is commonly modeled as a stationary covariance function when the underlying stochastic process of a function ($f$) is assumed to be stationary. A stationary isotropic covariance function that is motion and translation invariant can be defined as $k(x_i,x_j) = \phi(\frac{d^2(x_i,x_j)}{l^2})$. Here, $\phi:[0,\infty ) \rightarrow \mathbb{R}$ is a positive definite function, $l^2$ is the CLS, and $d(x_i,x_j)$ is the distance function (in geophysical measurements, this is commonly angular distance). As can be seen from the construction, when the CLS is large the correlation between the input space points is spatially small, and vice versa. In other words, the points that are far apart have a small effect on the inference step. ~\looseness-1

When the CLS is differing with respect to the input space dimensions, $d$, the covariance function is called an anisotropic covariance function. One such form of this covariance function is given by $k(x_i,x_j) = \phi(d(x_i,x_j)^T(\Sigma)^{-1}d(x_i,x_j))$. Here, the CLS, $\Sigma : \mathbb{R}^d \rightarrow S^+(d,\mathbb{R})$, can be decomposed into $\Sigma = \Gamma^TD\Gamma$. $\Gamma$ is a $d \times k$ column vector and it gives the direction of relevance, $D$ is a $k \times k$ diagonal matrix and it gives the magnitude of relevance, and $k$ is the relevant number of axis-aligned dimensions. 
 
To date, various $\phi$ for $k(\cdot,\cdot)$ have been proposed. For spatial statistics, the Mat\'{e}rn covariance function is a standard choice due to its flexibility in capturing the varying smoothness of the underlying distribution:~\looseness-1

\vspace{-1em}
\begin{flalign}\label{eq:mat}
k(x_i,x_j) = K_S(\nu,\sqrt{Q_{ij}}) = \frac{{(\sqrt{2\nu Q})}^{\nu} K_{\nu}(\sqrt{2\nu Q})}{{\sigma_f}^{-2} 2^{\nu-1} \Gamma(\nu)},
\end{flalign}

\noindent where $Q_{ij} = \frac{d^2{(x_{i},x_{j})}}{l^2}$, $\nu$ is the smoothness hyperparameter that controls the differentiability of the function, $\sigma_f$ is the signal variance, $K_\nu$ is the modified Bessel function, and $\Gamma(\nu)$ is the gamma distribution.

\subsection{Non-stationary covariance function}\label{sec:nsgp}
The non-stationary covariance function can be modeled using a covariance function that varies its CLS with respect to the input point of interest. The non-parametric model form for such a covariance function is given by  $k(x_i,x_j) = \phi(Q_{ij})$, $Q_{ij} = d(x_i,x_j)^T(\frac{\Sigma_i + \Sigma_j}{2})^{-1}d(x_i,x_j)$, and $\Sigma_i := \Sigma(x_i) =  \Gamma(x_i)^TD(x_i)\Gamma(x_i)$. For this non-stationary covariance function to be valid, \cite{higdon1999non} shows (using a convolution of kernels) that $\Sigma(x) : \mathbb{R}^d \rightarrow S^+(d,\mathbb{R})$ should be spatially evolving smooth kernels of symmetric positive definite matrices. 

A closed-form solution of the nonstationary Mat\'{e}rn covariance function is given by \cite{paciorek2004nonstationary} as:~\looseness-1

\vspace{-1em}
\begin{flalign}\label{eq:GloblNSGPcov}
k(x_i,x_j) = K_{NS}(x_i,x_j) = \frac{2^{d/2}\sigma_f^{2}{|\Sigma_i|}^{\frac{1}{4}}  {|\Sigma_j|}^{\frac{1}{4}}} {{|\frac{{\Sigma_i + \Sigma_j}}{2}|}^{\frac{1}{2}}} K_S(\nu,\sqrt{Q_{ij}}).
\end{flalign}

Intuitively, Equation \eqref{eq:GloblNSGPcov} indicates that the covariance between two observed values ($y_i$) is affected by the convolution (arithmetic average) of the local covariates (kernels) at the input locations ($x_i$). Consequently, this results in a variable CLS at the target values ($y^*$). ~\looseness-1

\subsection{Spatially evolving kernels for the characteristic length scales}
For modeling the CLS,~\cite{higdon1999non} uses a predetermined area of the ellipse,~\cite{fuglstad2013non} uses B-splines, and~\cite{paciorek2004nonstationary} uses a stationary GP. All of these methods capture the underlying local correlation of the data using $N$ neighbors of $x_i$. In this paper, we are interested in models that ensure a locally smooth manifold of $\Sigma_i$, which lies in the space of positive definite matrices $S^+(d,\mathbb{R})$ in a non-parameteric fashion, while still being entirely in a GP framework.~\cite{paciorek2004nonstationary} proposes the following eigen-decomposition of $\Sigma_i$ for the spatial data: ~\looseness-1

\vspace{-1em}
\begin{equation}\label{eq:Ker}
\Gamma(x_i) = 
\begin{bmatrix}
	\frac{u}{l_{uv}} & \frac{-v}{l_{uv}}          \\
	\frac{v}{l_{uv}} & \frac{u}{l_{uv}}         \\
 \end{bmatrix}, D(x_i) = \begin{bmatrix}
              \log(\lambda_1) &   0       \\
              0 & \log(\lambda_2)        \\
            \end{bmatrix}, l_{uv} = \sqrt{ u^2 + v^2}.
\end{equation}

The eigen decomposition parameters $\{{u,\lambda_1}\}_i$ are each, in turn, modeled using an independent stationary GP in the latitude dimension of $x_i$, and  $\{{v,\lambda_2}\}_i$ using an independent stationary GP in the longitude dimension of $x_i$. All four independent GPs have their own sets of hyperparamters. 

\section{Intrinsic non-stationary covariance function}

\begin{figure*}[t]
\centering
    \includegraphics[width=.8\linewidth]{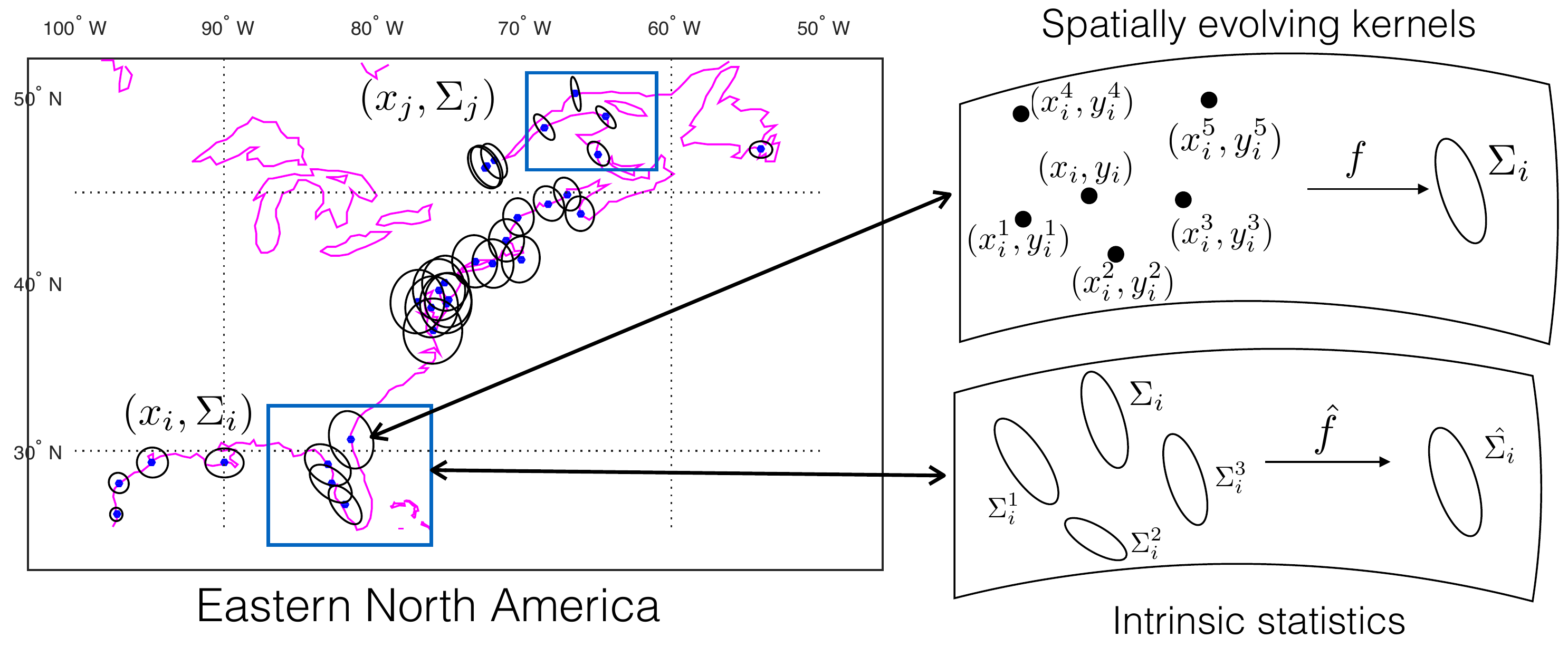}
\caption{Proposed framework for the intrinsic non-stationary covariance function.} \label{fig:Frame} 
\vspace{-1.5em}
\end{figure*}

In this section, we first propose our model of the covariance function (intrinsic non-stationary covariance function), then give background for the intrinsic statistics that we employ to construct the intrinsic non-stationary covariance, and finally provide an algorithm for its implementation in kriging.

In order to improve the model of the non-stationary covariance function that captures the variable regional information, while still maintaining a non-parametric model, we propose a new class of covariance functions and define this class as an intrinsic non-stationary covariance function. 

\begin{definition}\label{def:iNSGP}
An intrinsic non-stationary covariance function $k(\cdot,\cdot)$ assumes the form $k(x_i,x_j) = \phi(Q_{ij})$, where $Q_{ij} = d(x_i,x_j)^T(\psi_{ij})^{-1}d(x_i, x_j)$ and $\psi_{ij} := \psi(\Sigma_i, \Sigma_j)$.

Here, $\psi: S^+(d,\mathbb{R}) \rightarrow S^+(d,\mathbb{R})$ is an objective function of the form:
\begin{equation}\label{eq:psi_iNSGP}
	\psi(\Sigma_i,\Sigma_j) := \underset{\bar{\Sigma}} {\mathrm{argmin}}~ \sum_{\{ i,j \} \in \cal{N}} d^2(\bar{\Sigma}, \Sigma_{\{ i,j \}}),
\end{equation}

where $\bar{\Sigma}$ is the CLS for the intrinsic non-stationary covariance function, $d^2(\cdot,\cdot)$ is the intrinsic distance metric, $\cal{N}$ is the neighbors (including itself) of the geospatial point of interest ${i,j}$, and $d(x_i,x_j)$ is the distance in the input space.

\end{definition}
 
In our study, the aim of the objective function $\psi(\cdot,\cdot)$ is to model the regional geophysical CLS of the underlying stochastic process. Intuitively, $\bar{\Sigma}$ represents the intrinsic mean of the CLS at the input points $\{ x_i,x_j \}$, and it incorporates the regional CLS in its convolution step for the intrinsic covariance function.

There are two important factors in modeling this function $\psi$ :
\begin{enumerate}[leftmargin=*]
	\item Finding the neighborhood points that represent the regional information.
	\item Finding a representative of the CLS that describes the statistics (up to second-order) for the region of interest.
\end{enumerate}

Figure~\ref{fig:Frame} depicts these two factors, and Section~\ref{sec:method} gives one such method to construct these classes of functions.

\begin{theorem}
The intrinsic non-stationary covariance function, as defined above, is a valid non-stationary covariance function.
\end{theorem}
\begin{proof}~\cite{higdon1999non} shows that the construction of a covariance function using a moving average specification (i.e., convolution of kernels) leads to a non-stationary process. In our construction (Equation~\ref{eq:psi_iNSGP}) we specify the convolution of kernels on a smooth manifold, preserving the properties of symmetric positive definite matrices, and, in turn, obtaining a non-stationary process definition of [6] with a positive definiteness of the covariance function.
\end{proof}

\subsection{Intrinsic Statistics for the characteristic length scales}

To compute the objective function $\psi$ for the space of spatially smooth kernels (${\{ \Sigma_i \}}_{i=1}^N$), we assume that $\psi$ lies on the Riemannian manifold $\mathcal{M}$ of positive definite matrices $S^+(d,\mathbb{R})$. The Riemannian manifold is assumed to be geodesically complete and endowed with a canonical affine connection.
\cite{calvo1991explicit} and~\cite{lenglet2006statistics} derived explicit forms of the geodesic distance on this manifold as: $d(\Sigma_i, \Sigma_j) = \sqrt{\frac{1}{2}\sum_{i=1}^{d}\log^2(\eta_i)}$, where $\eta_i$ denotes the $d$ eigenvalues of the matrix $ (\Sigma_i^{{-1}/{2}}\Sigma_j\Sigma_i^{{-1}/{2}}) \in S^+ $.

\cite{frechet1948elements} defined the empirical Riemannian mean $\bar{\Sigma} \in S^+(d,\mathbb{R})$, and \cite{karcher1977riemannian} gave an iterative form of $\bar{\Sigma}$ as a local minimum of the objective function $\lambda^2 : S^+(d,\mathbb{R}) \rightarrow \mathbb{R}^+$:

\vspace{-1em}
\begin{equation} \label{eq:mean}
\lambda^2(\Sigma_1,\ldots ,\Sigma_{{N}}) = \frac{1}{{N}} \sum_{k=1}^{{N}} d^2(\Sigma_k, \bar{\Sigma}).
\end{equation} 

Each of the $N$ normal distributions $p(.|\Sigma_k)$ is associated with a unique tangent vector $\beta_k \in S(d,\mathbb{R})$, such that $\bar{\Sigma}$ is mapped onto $\Sigma_k$ by an exponential map $\exp_{\bar{\Sigma}}(\beta_k) = \bar{\Sigma}^{1/2} \exp(\bar{\Sigma}^{-1/2}\beta_k\bar{\Sigma}^{-1/2}) \bar{\Sigma}^{1/2}$. The covariance matrix of a set of covariance matrices itself (i.e., in our model this is a covariance of a set of the CLS) is then defined as:

\vspace{-1em}
\begin{equation} \label{eq:var}
\Lambda_{\bar{\Sigma}} = \frac{1}{N-1} \sum_{k=1}^{N}\beta_k\beta_k^T.
\end{equation}

The covariance of the set of the CLS gives us a measure by which to evaluate the properties of its distribution on a manifold. For example, when the $tr(\Lambda)$ is small, it means that the set of covariances (CLS) are from the same normal distribution $\mathcal{N}(\Sigma|\bar{\Sigma}, \Lambda)$ and are highly correlated.

Note the similarity in the variable $N$ of \eqref{eq:mean} and $\{ \{ i, j \} \in \mathcal{N} \}$ of \eqref{eq:psi_iNSGP}, which shows that $\psi_{ij}$ is in fact the intrinsic Riemannian mean $\bar{\Sigma}$.

\subsection{Algorithm for the intrinsic non-stationary covariance function}\label{sec:method}

{\footnotesize
\begin{algorithm}
\caption{Kriging for the Intrinsic Non-stationary Covariance Function}
\label{alg:method}
\begin{algorithmic}[1]
\STATE  {\bfseries Input:} Training set $\{(x,y)\}_{i=1}^n$, Test points $\{x^*\}_{i=1}^m$.
\STATE {\bfseries Output:} Predictive densities $p(y^* | x^*)$.
\STATE Initialize estimates of $\Sigma_i$ using Equation \eqref{eq:Ker}.
\STATE Find the characteristic length scale $\psi_{ij}$ using Algorithm \ref{alg:KNN}.
\STATE Compute the intrinsic nonstationary covariance function using Equation \eqref{eq:GloblNSGPcov}.
\STATE Compute the GP using $p(y^* | x^*, f) \sim \mathcal{N}(\mu,\sigma)$.
\end{algorithmic}
\end{algorithm}
}

Algorithm \ref{alg:method} describes our framework for implementing the intrinsic non-stationary covariance function, and Figure \ref{fig:Frame} depicts this general framework. We first obtain initial smooth estimates of the CLS ($\Sigma_i$) by Equation \eqref{eq:Ker}, and then we update $\Sigma_i$ with the aim of modeling the intrinsic function $\psi_{ij}$ using our approach proposed in Algorithm \ref{alg:KNN}. Our method maintains the appropriate smoothness in the latent space of $\Sigma_i$ due to the second-order intrinsic statistics on the Riemannian manifold. Additionally, our  method captures the correlation of the $\Sigma_i$ in its intrinsic latent space that the initial estimate failed to capture. Hence, the estimates for the CLS around sharp discontinuities and separated regions of the spatial data field are improved, as shown in Figure \ref{fig:Ker}. Finally, the CLS ($\psi_{ij}$) is used in Equation \eqref{eq:GloblNSGPcov} for the GP regression model.~\looseness-1

{\footnotesize
\begin{algorithm}
\caption{Intrinsic Characteristic Length Scale}
\label{alg:KNN}
\begin{algorithmic}[1]
\STATE  {\bfseries Input:} $\{ x_i,x_j \}$, $\{x_k,\Sigma_k\}_{k=1}^n$.
\STATE {\bfseries Output:}  $\psi_{ij}$.
\STATE Find K nearest neighbors ${\cal N}_{i}$ for $x_i$ and ${\cal N}_{j}$ for $x_j$ in the input space.
\STATE	Remove the $\Sigma_i$'s from ${\{ \Sigma_i \}}_{i \in {\cal N}_i }$, such that, $| tr(\Lambda_i) |$ $<$ threshold.
\STATE Find ${\bar{\Sigma}}_i$ using $\lambda^2({\{ \Sigma_i \}}_{i \in {\cal N}_i })$ for Eq.~\ref{eq:mean}.
\STATE	Remove the $\Sigma_j$'s from ${\{ \Sigma_j \}}_{j \in {\cal N}_j }$, such that, $| tr(\Lambda_j) |$ $<$ threshold.
\STATE Find ${\bar{\Sigma}}_j$ using $\lambda^2({\{ \Sigma_j \}}_{j \in {\cal N}_j })$ for Eq.~\ref{eq:mean}.
\STATE Find $\psi_{ij}$ using $\lambda^2({\bar{\Sigma}}_i, {\bar{\Sigma}}_j)$ 
\end{algorithmic}
\end{algorithm}
}

For step (3) of Algorithm \ref{alg:method},~\cite{paciorek2004nonstationary} bound the ${\{ \Sigma_i \}}_{i=1}^{N}$'s to achieve the required smoothness and used the arithmetic mean to convolve $\Sigma_i$ and $\Sigma_j$ for its covariance function. We used the empirical Riemannian mean rather than the arithmetic mean and, therefore, are free from fixing the bounds on the initial estimates of ${\{ \Sigma_i \}}_{i=1}^{N}$. This allows the variable smoothness in Equation \eqref{eq:GloblNSGPcov} to naturally evolve in the space of symmetric positive definite matrices. Note, the thresholds of variance are bounded in the space of ${\{ \Sigma_i \}}_{i=1}^{N}$ and are not directly dependent on the input space. Additionally,~\cite{arsigny2006log} shows that the arithmetic mean causes larger determinants (than the original determinant) in the space of $S^+(d,\mathbb{R})$, which the Riemannian mean avoids.~\looseness-1

Algorithm \ref{alg:KNN} uses two levels of nearness measures: 1) the usual distance metric directly on the input space that measures spatial proximity, and 2) the intrinsic statistics of $\Sigma_i$ that measure its proximity on the manifold. Here, we used $| tr(\Lambda_i) |$ to measure the correlation of the neighboring $\Sigma_i$'s for its simplicity, but one could also use the K-nearest neighbors in the space of $\Sigma_i$s. It would be worth exploring how the different statistical measures on the manifold of the CLS space could improve the kriging estimates when sharp jumps exist in the underlying distribution.~\looseness-1

For example, a CLS computation of a geospatial location that is close to the boundary of its geophysical region is less reliable when it is the function of its input space alone (i.e., $\Sigma_i$ in Algorithm~\ref{alg:method}, Step 3). However, when the additional information of the neighboring CLS is incorporated into the model (i.e., Algorithm~\ref{alg:KNN}, Step 4 and 5), one could potentially recover the CLS of the boundary points that is closer (Equation~\ref{eq:mean}) to the CLS representative of its associated region. We show in the experimental section that this approach of Algorithm \ref{alg:KNN} is particularly useful for analyzing climate models. ~\looseness-1

For the initialization of the CLS in Algorithm~\ref{alg:method}, Step (3), one can use the numerical implementation of either~\cite{paciorek2004nonstationary},~\cite{higdon1999non}, or~\cite{fuglstad2013non}. In our experiments, we used~\cite{paciorek2004nonstationary}, because the empirical testing showed that it gave the best results for our application. Similarly, we used the numerical  implementation of~\cite{lenglet2006statistics} for the intrinsic statistics (Equations \eqref{eq:mean} and \eqref{eq:var}). ~\looseness-1

The K for the nearest neighbors and the threshold for the variance in Alg.~\ref{alg:KNN} are found using the 5-fold Cross Validation. To find the GP hyperparameters, for step (1) and step (5) in Alg.~\ref{alg:method}, we use a Markov Chain Monte Carlo (MCMC) sampling scheme that is similar to the one outlined in~\cite{plagemann2008nonstationary}.~\looseness-1

\vspace{-0.5em}

\section{Experiments on climate related data}

To gain insight into the applicability of our proposed covariance function, we implemented and compared kriging with three different covariance functions: the widely used stationary anisotropic Mat\'{e}rn covariance function (statGP) of~\cite{matern1960spatial}, the baseline non-stationary covariance function (NSGP) of~\cite{paciorek2004nonstationary}, and the intrinsic non-stationary covariance function (iNSGP) that we propose in this paper. These methods were evaluated using two standard performance measures (as described below) for kriging. The three datasets used are: the smooth 2$d$ simulated dataset (SIM) given in ~\cite{paciorek2004nonstationary}, the geophysics driven synthetic\footnote{\url{http://www.psmsl.org/train\_and\_info/geo\_signals/gia/peltier/}} data (GIA) as modeled in~\cite{peltier2004global}, and the global-scale complex naturally occurring real\footnote{\url{http://www.psmsl.org/data/obtaining/}} dataset (TG) collected in~\cite{holgate2012new}. ~\looseness-1

\textit{The Simulated Dataset}. For this study, we are interested in comparing our method with the 2$d$ simulated functions that have been previously used in the non-stationary covariance function literature. The experimental set up is given in~\cite{paciorek2004nonstationary}. Even though the simulation function is non-stationary, it is fairly smooth and homogeneous, and it lacks the complex regional geophysics that is usually encountered in global-scale climate related data.~\looseness-1

\textit{The Synthetic Dataset}. For this study, we are interested in modeling the global-scale geophysical signal that is present in the climate datasets (such as tide gauge data). This synthetic dataset enables us to compare the three covariance functions (statGP, NSGP, iNSGP) with known geophysics boundaries.~\looseness-1

One such widely modeled geophysical signal is the Glacial-isostatic adjustment (GIA) model. We used the signal modeled in GIA~\cite{peltier2004global}, which gives the difference in the height between the sea surface and solid earth. Figure \ref{fig:Orig}(a) shows the distribution of the underlying distribution ($y$) after masking out the land mass. To examine the robustness of our proposed model, the experimental setup includes 25 independent runs of the GIA data. For each run, we randomly sampled 315 training points, added Gaussian noise ($\eta=0.2$) to sea level measurements, and tested on 946 independent random sampled points.~\looseness-1

\textit{The Real Dataset}. For this study, we are interested in a global-scale climate variable that contributes to future climate related risks, is known to have non-stationarity, and has complex regional geophysics in its dataset. Hence, we used tide gauge sites measuring relative sea level measurements across the coastal areas of the globe (see Figure \ref{fig:Orig}(b)). From the dataset given in~\cite{holgate2012new}, 747 locations were selected because they had consistent temporal records from the years 1993 to 2012. To focus our study on the geo-spatial set up, we used 747 locations of tide gauge sites to construct the rate of change of the sea level (mm/year), where the annual sea level rate of change was obtained from linear regression estimates in the temporal dimension. The experimental set up then includes 25 independent runs with randomly sampled 374 and 373 locations for respectively the training and test set. ~\looseness-1

\textit{Evaluation metrics}. We report performance with respect to two widely adopted metrics in kriging: the standardized mean squared error (sMSE) and the negative log predictive density (nLPD). sMSE measures the point estimate errors in the predictions and is given by: $sMSE = n^{-1}\sum_{i=1}^{n} \mathrm{var}(y)^{-1} (y_i - y^*_i)^2 $, where $n$ is the number of test points, $y^*_i$ is the predictive mean at input space $x_i$, and $\mathrm{var}(y)$ is the sample variance. nLPD measures not just the point estimates error, but also the error variance of the predictions and is given by: $nLPD = - n^{-1}\sum_{i=1}^{n}\log(p(y_i | x_i))$.~\looseness-1

\begin{table}[t]
\begin{center}
\caption{Evaluations of the simulated (SIM), synthetic (GIA),  and real (TG) datasets. GIA and TG dataset results are in the units of mm/year.}
{
\begin{tabular}{| c | c | c | c | c | c | c | c | c |}	\hline
	Methods 	& 	\multicolumn{2}{|c|}{SIM}	& 	\multicolumn{2}{|c|}{GIA (All)}	&  \multicolumn{2}{|c|}{GIA (Reg.1)}	& 	\multicolumn{2}{|c|}{TG (All)}	\\ 
  			
  				&	sMSE		&	nLPD			&	sMSE		&	nLPD			&	sMSE	&	nLPD					&	sMSE		&	nLPD			\\ \hline
  	StatGP		&	0.0240	&	0.311			&	0.58 		&	3.08			&	1.57	&	19.23					&	0.85		&	2.81 			\\
  	NSGP		&	0.0237	&	0.278			&	0.56 		&	2.00			&	1.24	&	7.30					&	0.75		& 	2.82			\\
  	iNSGP		&\textbf{0.0235}	&	\textbf{0.271}			&	\textbf{0.54} &	\textbf{1.94}	&	\textbf{1.03} & \textbf{6.90}		&	\textbf{0.71} &	\textbf{2.78}	\\ \hline
\end{tabular}
}
\label{tab:Eval}
\vspace{-1.5em}
\end{center}
\end{table}

\vspace{-0.8em}

\section{Results}
\label{sec:res}

\begin{figure*}[t]
\centering
        \includegraphics[width=.9\linewidth]{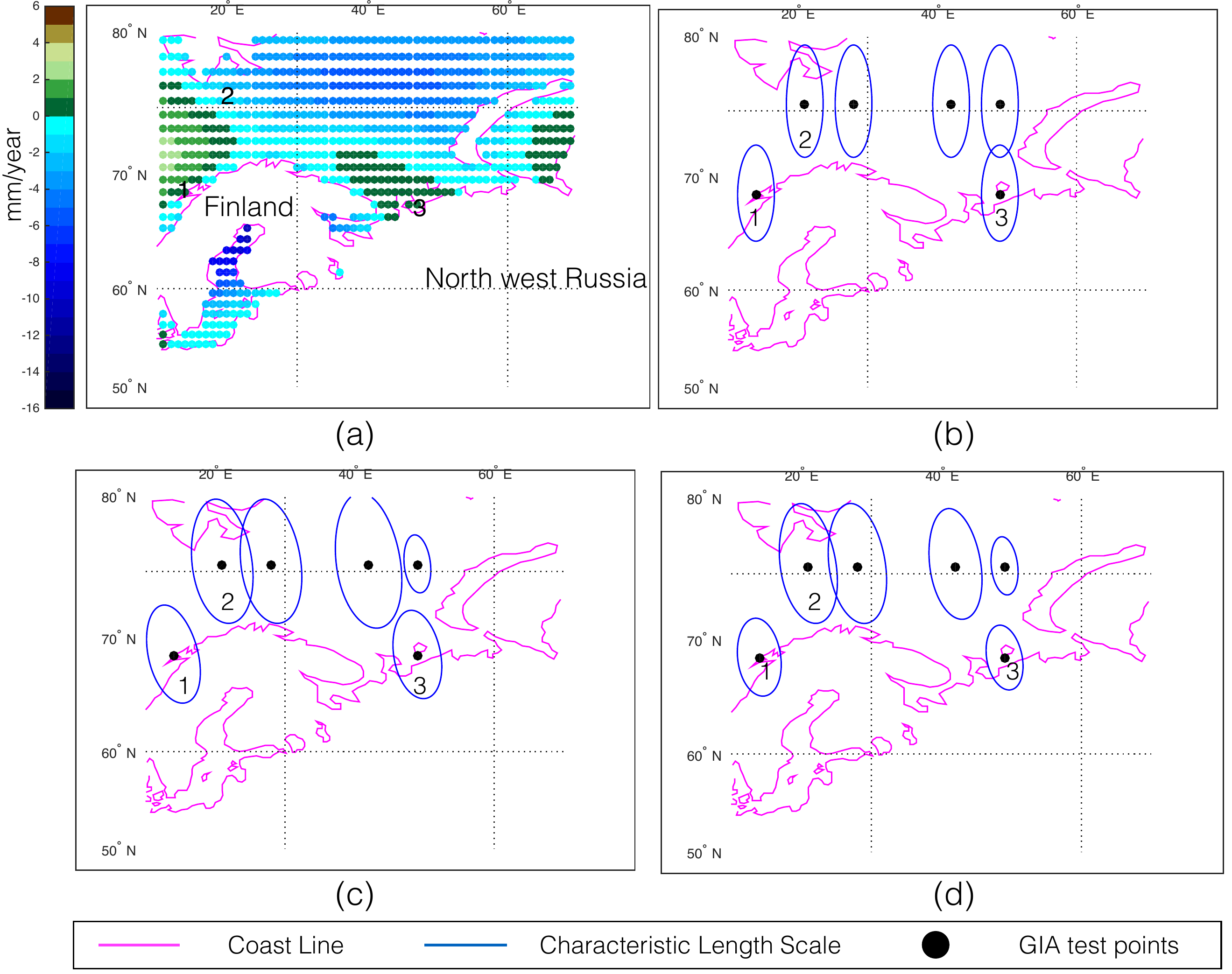}
\caption{GIA dataset around Barrent sea (Reg.1) (a) GIA data (true values), (b) Stationary GP (statGP) characteristic length scale (CLS), (c) Non-stationary GP (NSGP) CLS, and (d) Intrinsic non-stationary GP (iNSGP) CLS.} 
\label{fig:Ker} 
\vspace{-1.5em}
\end{figure*}

Table~\ref{tab:Eval} summarizes the performance for the three covariance functions (statGP, NSGP, and iNSGP) when implemented on the three datasets (SIM, GIA, and TG). The performance of iNSGP is particularly improved for the GIA and TG datasets, while it does not show much of an improvement over NSGP for the SIM dataset. This is mainly because the SIM dataset is fairly smooth. Furthermore, the SIM data does not suffer from the issue of a non-uniformly smooth spatial boundary (regional geophysics) that is present in GIA and TG datasets Fig~\ref{fig:Orig}(a).~\looseness-1

For example, Fig~\ref{fig:Ker}(a) shows the true values of GIA data near the Barent sea (Reg.1), which is a marginal sea of the Arctic ocean. The CLS estimates for this region (which is parametrically plotted as ellipses) from statGP, NSGP, and iNSGP are shown in Fig~\ref{fig:Ker}(b,c,d). Note the differences in the shape of the ellipses numbered (1,2,3) in Fig~\ref{fig:Ker}(b,c,d). These points (1,2,3) correspond to the geospatial boundary points of the regions (1,2,3) in Fig~\ref{fig:Ker}(a). While statGP has the same shape for all of the points, NSGP shows some variation in its shape. Even so, they are largely similar. On the other hand, iNSGP has a distinct variation in the the points (1,2,3) and well represents the differences in these three regions. This is explained by the superior performance (Tab.~\ref{tab:Eval}) of iNSGP over NSGP and statGP (particularly in the GIA Reg.1).~\looseness-1

Table~\ref{tab:Eval}, shows the average values for the error metrics. The maximum standard deviation in sMSE for SIM was 0.003, for GIA was 0.02, and for TG was 0.5. For each of these runs, iNSGP performed as well as NSGP for the SIM dataset, and outperformed NSGP and statGP for the TG and GIA datasets. The error values for the SIM dataset when applying the methods of StatGP and NSGP were similar to the error values found in~\cite{paciorek2004nonstationary}.~\looseness-1

From all three datasets, the real dataset (TG) shows the most improvement in its error metric when applying the iNSGP method. The high improvement is because TG has large variability in the regional geophysics; therefore, the intrinsic non-stationary covariance function is able to better model the underlying true distribution than the stationary and non-stationary covariance functions.~\looseness-1

\section{Discussion and Concluding Remarks}

We introduced a new class of covariance functions, which we call an intrinsic non-stationary covariance function. This covariance function is especially useful in modeling global scale geospatial datasets that have a non-stationary process and non-uniformly smooth spatial boundaries due to regional geophysics. We developed a framework to apply this covariance function for kriging. Using the sea level dataset from the Glacial-isostatic adjustment model and tide gauge measurements, we demonstrated our framework's improved performance in kriging when compared with the non-stationary covariance function.~\looseness-1

There are many facets of implementation of the intrinsic non-stationary covariance function that could be undertaken in the future. One of the important issues, especially outside of the geo-spatial community, is large datasets. Specifically, one can use sparse regression techniques~\cite{lawrence2003fast} and a computationally cost effective metric for the intrinsic statistics on the positive definite matrices of the characteristic length scale~\cite{arsigny2006log} to deal with large datasets.~\looseness-1

Other climate related variables, such as temperature and precipitation records, face similar modeling issues as the sea level dataset that we used for our application. Future work will explore the application and methods of our intrinsic covariance function to such geospatial datasets, with the larger goal of aiding in the assessment of future climate related risks.~\looseness-1


\bibliographystyle{unsrt}
\small{ 
\bibliography{iNSGPcad_15a}
}

\end{document}